\documentclass[letterpaper]{article} 
\usepackage[draft]{aaai2027}  
\usepackage[hyphens]{url}  
\usepackage{graphicx} 
\usepackage{natbib}  
\usepackage{caption} 
\usepackage{amsmath}
\usepackage{amssymb}
\usepackage[table]{xcolor}
\usepackage{booktabs}
\usepackage{amsthm}
\usepackage{algorithm}    
\usepackage{algorithmic}  
\usepackage{multirow}
\usepackage{makecell}
\definecolor{mygray}{gray}{0.92}
\definecolor{oursblue}{RGB}{232,241,250}
\newcommand{\ours}{\textbf{GOTS}}
\newcommand{\OursRowColor}{\rowcolor{oursblue}}

\newtheorem{proposition}{Proposition}

\title{GOTS: Greedy Orthogonal Token Selection for High-Resolution Vision-Language Models}
\author{
Jun Ling,
Tao Huang,
Junzhuo Liu,
Bowen Tang,
Peng Wang
}

\affiliations{
School of Computer Science and Engineering,\\
University of Electronic Science and Technology of China\\
\texttt{cs.lingjun@gmail.com},
\texttt{p.wang6@hotmail.com}\\
}

\begin{document}

\maketitle
\begin{abstract}

Modern vision-language models (VLMs) increasingly rely on dynamic or high-resolution visual encoding, producing thousands of visual tokens that substantially increase downstream language-model inference cost. Existing token-reduction methods assess token utility through token-wise importance, query relevance, coverage, pairwise diversity, or subset-level objectives. Our key insight is to view visual token reduction through selected-span complementarity: instead of scoring a token in isolation or through pairwise relations, we assess how much of its feature is orthogonal to the span of the already retained subset. Based on this perspective, we propose Greedy Orthogonal Token Selection (GOTS), a training-free and query-agnostic method. At each step, GOTS selects the token with the largest residual energy orthogonal to the current retained span. This rule exactly maximizes the one-step augmented Gram determinant among candidate additions, giving each greedy step a precise local geometric guarantee for subset expansion.  Across five high-resolution VLM backbones from the Qwen-VL and InternVL families and eleven diverse benchmarks, GOTS achieves higher average performance retention than the strongest evaluated baselines, and a controlled OCRBench study shows that it reduces model-side time-to-first-token after accounting for selection overhead. Code is available at \url{https://github.com/newLLing/GOTS}.

\end{abstract}

\section{Introduction}

Modern vision-language models (VLMs) increasingly adopt dynamic or high-resolution visual encoding to capture fine-grained evidence for tasks such as optical character recognition, document understanding, chart interpretation, and spatial reasoning~\citep{Wang2024Qwen2VLEV,Bai2025Qwen3VLTR,Bai2025Qwen25VLTR,Wang2025InternVL35AO}. However, this capability often produces thousands of visual tokens for a single image, whose quadratic self-attention cost in the downstream language model substantially increases prefill latency and activation memory.

We study training-free visual token reduction in this high-resolution regime. Selection indices are computed from vision-encoder features, keeping only the selected visual tokens for subsequent large language model (LLM) inference, so discarded tokens incur no downstream LLM computation. We focus on dynamic or high-resolution VLMs that generate long, input-dependent visual sequences, where this bottleneck is particularly pronounced.

Existing methods estimate token utility using visual attention, query relevance, coverage, or feature diversity. VisionZip~\citep{yang2025visionzip} retains attention-dominant tokens and merges the remainder by feature similarity, while DivPrune~\citep{alvar2025divprune} favors tokens with large pairwise distances. MMTok~\citep{dong2025mmtok} formulates selection as a subset-level coverage objective over query and visual tokens, whereas CDPruner~\citep{zhang2025cdpruner} uses a query-conditioned determinantal point process (DPP) to balance relevance and diversity. GOTS complements these token-wise, pairwise, coverage, and determinant objectives by directly using the contribution beyond the collective retained span as its primary selection criterion.

Our key insight is to view visual token reduction through selected-span complementarity. Rather than scoring a token only in isolation or through pairwise relations, we assess the feature component it contributes beyond the retained span. Under a linear approximation, this contribution can be quantified by the candidate's projection residual. Selected-span complementarity captures a subset-level relation: a candidate separated from each retained token can still substantially overlap with directions represented by their collective span. Coverage and DPP objectives model subsets through coverage functions or query-conditioned kernels; GOTS measures contribution beyond the collective span.

Motivated by this perspective, we formulate visual token reduction as a row-subset selection problem and propose GOTS, a training-free and query-agnostic approach for high-resolution VLM inference. GOTS measures each candidate by its residual component outside the currently selected span, greedily retains the token with the largest residual energy, and updates the remaining residuals through orthogonal projection. The selected indices are restored to their original order, determining which visual tokens remain for LLM inference. GOTS requires neither additional training nor access to the textual query and does not explicitly construct an $N\times N$ token-similarity kernel.

This selection principle has a precise geometric interpretation. The Gram determinant of a selected feature subset measures the squared volume spanned by its representations and vanishes under linear dependence. We show that, under nonzero pivot residuals, adding a candidate multiplies the determinant by that candidate's residual energy relative to the retained span. Consequently, each GOTS iteration maximizes the Gram determinant among one-token augmentations of the current subset. This establishes an exact one-step characterization of selected-span complementarity and defines the geometric guarantee provided by the proposed greedy selection rule.

We evaluate GOTS on eleven benchmarks covering general VQA, OCR, document and chart understanding, perception, spatial reasoning, and mathematical reasoning, using five high-resolution VLM backbones from Qwen-VL and InternVL families (7B--32B parameters). In the primary Qwen2.5-VL-7B comparison, GOTS achieves the highest average retention at three token budgets. Across four additional backbones, it maintains an average-retention advantage over the strongest baseline at every ratio, with margins that widen on average under more aggressive compression. A controlled OCRBench study shows a net model-side time-to-first-token (TTFT) reduction after accounting for selection overhead.

Our main contributions are summarized as follows:
\begin{itemize}
\item We identify selected-span complementarity as a visual-token selection principle, using residual energy relative to the retained feature span to quantify each candidate's contribution beyond the collectively retained representation.
\item We introduce GOTS, a training-free and query-agnostic selector that greedily retains the token with the largest residual energy. We establish that each selection exactly maximizes the next multiplicative Gram-determinant factor, providing a one-step geometric characterization without constructing an $N\times N$ similarity kernel.
\item We evaluate GOTS on eleven benchmarks and five high-resolution VLM backbones, showing the highest average retention in the primary comparison, cross-backbone advantages over the strongest evaluated baselines, and model-side TTFT reduction on OCRBench after selection overhead.
\end{itemize}

\section{Related Work}

\paragraph{High-Resolution Vision-Language Models.}
VLMs typically combine a vision encoder, a modality projector, and a pretrained language model. Earlier systems such as LLaVA~\citep{liu2024improvedbaselinesvisualinstruction} and InstructBLIP~\citep{dai2023instructblip} process images at fixed resolution with predetermined-length visual sequences. Recent Qwen-VL~\citep{Wang2024Qwen2VLEV,Bai2025Qwen3VLTR,Bai2025Qwen25VLTR} and InternVL~\citep{Wang2025InternVL35AO} models support dynamic or high-resolution encoding, enabling fine-grained evidence but producing longer, input-dependent sequences. GOTS targets this regime by computing selection indices from vision-encoder features to retain informative visual tokens for language-model inference.

\paragraph{Query-Agnostic Post-Encoder Token Reduction.}
Query-agnostic methods compress visual tokens independently of the textual instruction. VisionZip~\citep{yang2025visionzip} identifies attention-dominant tokens and merges the rest by feature similarity. DivPrune~\citep{alvar2025divprune} greedily maximizes the minimum pairwise distance among retained tokens. These criteria emphasize token-wise importance or pairwise separation. GOTS instead evaluates each candidate by the feature component it contributes beyond the collective retained span, using its orthogonal residual as the selection criterion. This formulation extends token-wise and pairwise criteria with a selected-span criterion that directly measures each candidate's complementary contribution to the retained representation.

\paragraph{Query-Conditioned Post-Encoder Token Reduction.}
Query-conditioned methods incorporate textual instructions into selection. MMTok~\citep{dong2025mmtok} formulates selection as submodular coverage maximization over query and visual tokens. CDPruner~\citep{zhang2025cdpruner} uses a query-conditioned determinantal point process (DPP) with a conditional-similarity kernel to model subset relevance and diversity. Both CDPruner and GOTS employ determinant-based subset geometry, but differ in conditioning, feature representation, and implementation. CDPruner evaluates a query-conditioned kernel determinant, whereas GOTS greedily expands the volume of unnormalized visual features via column-pivoted QR (QRCP), using no explicit $N\times N$ kernel. GOTS adapts this query-agnostic determinant geometry to visual-token selection.

\paragraph{Matrix Subset Selection.}
Column and row subset selection seek compact subsets of a matrix that preserve useful structure~\citep{Boutsidis2009AnIA}. Determinant-based objectives characterize selected-subset volume, while volume sampling provides expected reconstruction guarantees under its sampling distribution~\citep{Deshpande2006MatrixAV}. Exact maximum-volume selection is combinatorial~\citep{Goreinov1997Theory}, motivating greedy approximations. QRCP selects the column with the largest residual norm relative to previously selected columns~\citep{Businger1965Linear}. GOTS applies this classical principle to the transposed visual feature matrix and interprets its pivots as visual-token indices that greedily expand the retained span, yielding an exact local stepwise volume-expansion guarantee.

\begin{figure*}[t]
\centering
\includegraphics[width=\textwidth]{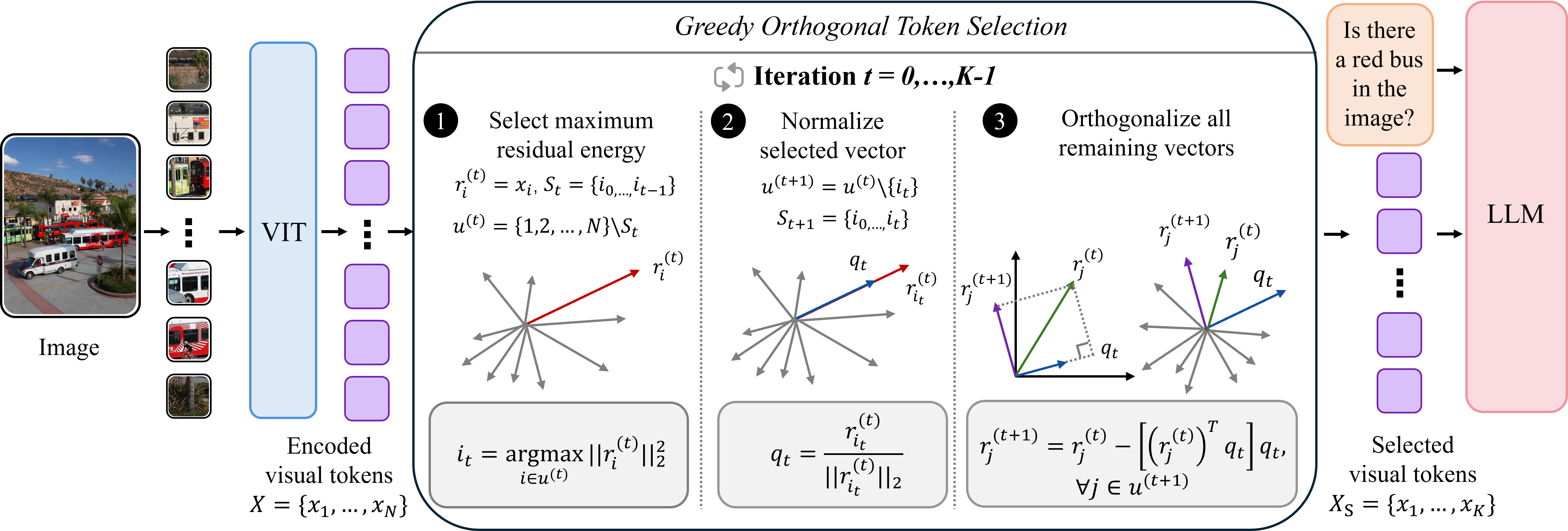}
\caption{\textbf{Overview of the proposed Greedy Orthogonal Token Selection.} The input image is encoded by a Vision Transformer (ViT); token selection is performed using the resulting vision-encoder features to retain informative visual tokens. In each iteration $t=0, \dots, K-1$, the selection process performs three key steps: (1) selecting the token with the maximum residual energy, (2) computing a normalized basis vector, and (3) updating the residuals of all remaining candidate vectors via orthogonal projection.}
\label{fig:method_overview}
\end{figure*}

\section{Method}

\subsection{Overview}
Given an input image encoded by a vision encoder, we obtain $N$ visual features corresponding to $N$ visual tokens; selection operates directly on these features and determines the retained token indices for subsequent inference (Figure~\ref{fig:method_overview}).
\begin{equation}
\mathbf{X} = [\mathbf{x}_1, \mathbf{x}_2, \ldots, \mathbf{x}_N]^{\top} \in \mathbb{R}^{N \times d},
\end{equation}
 where $\mathbf{x}_i \in \mathbb{R}^{d}$ is the vision-encoder feature row vector associated with the $i$-th encoded visual token, and $d$ is the feature dimension. The objective of visual token reduction is to select a compact index set $\mathcal{S} \subset \{1, \ldots, N\}$ with an output token budget $K$, where $K \ll N$ and typically $K \le d$, such that the selected feature matrix $\mathbf{X}_{\mathcal{S}} \in \mathbb{R}^{K \times d}$ preserves information useful to the downstream VLM.

Greedy Orthogonal Token Selection (\ours{}) consists of two steps. First, it applies the QRCP selection rule to $\mathbf X^\top$, greedily selecting the visual feature with the largest residual energy relative to the current selected span. Second, it sorts the selected indices into their original sequence order and forwards retained visual tokens together with their positional metadata to the VLM.

\subsection{Row Subset Selection and a Volume Surrogate}
From a linear approximation perspective, a natural row subset selection problem (RSSP)~\citep{Boutsidis2009AnIA} is to identify $K$ rows whose span minimizes the projection error of the complete visual feature matrix:
\begin{equation}
\mathcal{S}_{\mathrm{rec}}^{*}
= \operatorname*{argmin}_{\substack{\mathcal{S} \subseteq \{1,\ldots,N\}\\|\mathcal{S}|=K}}
\|\mathbf{X}-\mathbf{X}\mathbf{P}_{\mathcal{S}}\|_{F}^{2},
\label{eq:rssp_reconstruction}
\end{equation}
where $\mathbf{P}_{\mathcal{S}}=\mathbf{X}_{\mathcal{S}}^{+}\mathbf{X}_{\mathcal{S}}\in\mathbb{R}^{d\times d}$ is the orthogonal projector onto the row span of $\mathbf{X}_{\mathcal{S}}$, and $\mathbf{X}_{\mathcal{S}}^{+}$ denotes the Moore--Penrose pseudoinverse. Since the rows in $\mathcal S$ lie in their own span, Eq.~\eqref{eq:rssp_reconstruction} reduces to the total residual energy of the discarded features: \begin{equation} \|\mathbf X-\mathbf X\mathbf P_{\mathcal S}\|_F^2 = \sum_{i\notin\mathcal S} \|\mathbf x_i(\mathbf I-\mathbf P_{\mathcal S})\|_2^2. \label{eq:discarded_residual_energy} \end{equation} Thus residual energy outside the retained span directly quantifies the reconstruction gap of the subset; minimizing it over all $K$-subsets is combinatorial.

We therefore use the squared volume of the selected rows as a geometric surrogate for complementarity. Assuming $K\leq\operatorname{rank}(\mathbf X)$, we define the squared volume as:
\begin{equation}
V(\mathcal S)^2
=\det\!\left(\mathbf X_{\mathcal S}\mathbf X_{\mathcal S}^{\top}\right),
\qquad |\mathcal S|=K.
\label{eq:subset_volume}
\end{equation}
The determinant vanishes when the selected rows are linearly dependent and increases with both feature magnitude and directional complementarity. This motivates the surrogate objective:
\begin{equation}
\mathcal{S}_{\mathrm{vol}}^{*}
=\operatorname*{argmax}_{\substack{\mathcal{S}\subseteq\{1,\ldots,N\}\\|\mathcal{S}|=K}}
\det\!\left(\mathbf X_{\mathcal S}\mathbf X_{\mathcal S}^{\top}\right).
\label{eq:max_volume}
\end{equation}
Volume sampling provides expected reconstruction guarantees when a subset is sampled with probability proportional to Eq.~\eqref{eq:subset_volume}~\citep{Deshpande2006MatrixAV}. Our deterministic formulation uses this volume principle as a tractable surrogate. The volume objective measures complementarity among selected rows, while Eq.~\eqref{eq:rssp_reconstruction} measures projection error over all rows; together, they motivate selecting a compact subset that spans diverse and high-energy feature directions. We use volume as a selection surrogate and evaluate its downstream utility empirically.

\subsection{QRCP via Incremental Orthogonalization}
Exact maximum-volume subset selection is combinatorial~\citep{Boutsidis2009AnIA,Goreinov1997Theory}. We instantiate the classical QRCP rule on $\mathbf X^\top$ through incremental modified Gram--Schmidt orthogonalization.

By the base-times-height identity, augmenting an existing subset with a new token multiplies its squared volume by the squared orthogonal distance from that token to the current selected span. We call this factor the \emph{residual energy}. Selecting the candidate with the maximum residual energy therefore maximizes the one-step augmented Gram determinant.

We implement this rule through incremental orthogonalization. Let $\mathcal{S}_t=\{i_0,\ldots,i_{t-1}\}$ denote the selected indices after $t$ iterations, let $\mathcal{U}^{(t)}=\{1,\ldots,N\}\setminus\mathcal S_t$ denote the available indices, and let $\mathbf r_i^{(t)}\in\mathbb R^d$ denote the residual vector of candidate $i$ relative to the current selected span. Initially, $\mathcal S_0=\emptyset$, $\mathcal U^{(0)}=\{1,\ldots,N\}$, and $\mathbf r_i^{(0)}=\mathbf x_i$. We define the residual energy as:
\begin{equation}
e_i^{(t)} \triangleq \|\mathbf r_i^{(t)}\|_2^2,
\label{eq:residual_energy}
\end{equation}
which is the squared orthogonal distance from $\mathbf x_i$ to $\operatorname{span}(\mathbf X_{\mathcal S_t})$.

At each step $t \in \{0,1,\dots,K-1\}$, we apply the maximum-residual-energy rule:
\begin{equation}
i_t = \operatorname*{argmax}_{i \in \mathcal{U}^{(t)}} e_i^{(t)}
= \operatorname*{argmax}_{i \in \mathcal{U}^{(t)}} \|\mathbf r_i^{(t)}\|_2^2.
\label{eq:pivot_rule}
\end{equation}
After selecting $i_t$, we remove it from the candidate pool and normalize its residual to form the unit row vector $\mathbf q_t=\mathbf r_{i_t}^{(t)}/\|\mathbf r_{i_t}^{(t)}\|_2$. We then update each remaining residual by subtracting its projection onto $\mathbf q_t$:
\begin{equation}
\mathbf r_j^{(t+1)}=\mathbf r_j^{(t)}-
\left(\mathbf r_j^{(t)}\mathbf q_t^{\top}\right)\mathbf q_t,
\quad \forall j\in\mathcal U^{(t+1)}.
\label{eq:residual_update}
\end{equation}
In exact arithmetic, the updated residuals are orthogonal to the newly added basis direction.

For numerical stability, if the maximum residual energy $e_{i_t}^{(t)}$ falls below $\epsilon^2$ (e.g., $10^{-12}$), the remaining candidates contribute no numerically independent direction. The procedure terminates early, yielding $|\mathcal S_t|<K$.

\paragraph{Computational and Memory Complexity.}
Each rank-one residual update costs $\mathcal{O}(Nd)$, yielding $\mathcal{O}(KNd)$ time over $K$ selections. The cost is linear in $N$ when $K$ is fixed; with token-retention ratio $r$ and $K=rN$, it becomes $\mathcal{O}(rN^2d)$. Storing the residual features and selected basis vectors requires $\mathcal{O}(Nd+Kd)$ space, which is $\mathcal{O}(Nd)$ when $K\leq N$. Unlike selectors that explicitly materialize all pairwise similarities, \ours{} does not store an $N\times N$ matrix; such a matrix requires $\mathcal{O}(N^2)$ space, while its direct construction generally requires $\mathcal{O}(N^2d)$ computation.

\subsubsection{Exact One-Step Volume Characterization}
The following proposition states the precise guarantee associated with the maximum-residual-energy rule.

\begin{proposition}[QRCP Equivalence and Gram-Determinant Decomposition]
\label{prop:qrcp}
Let $\mathcal S_t=\{i_0,\ldots,i_{t-1}\}$ be the indices selected after $t$ iterations, and assume that the pivot residual energies are nonzero. Then, the selected-set Gram determinant decomposes as:
\begin{equation}
\det\!\left(\mathbf X_{\mathcal S_t}\mathbf X_{\mathcal S_t}^{\top}\right)
=\prod_{\tau=0}^{t-1}e_{i_\tau}^{(\tau)}
=\prod_{\tau=0}^{t-1}\left\|\mathbf r_{i_\tau}^{(\tau)}\right\|_2^2.
\label{eq:volume_product}
\end{equation}
Moreover, for every remaining candidate $j\in\mathcal U^{(t)}$, the augmented Gram determinant satisfies:
\begin{equation}
\det\!\left(\mathbf X_{\mathcal S_t\cup\{j\}}\mathbf X_{\mathcal S_t\cup\{j\}}^{\top}\right)
=\det\!\left(\mathbf X_{\mathcal S_t}\mathbf X_{\mathcal S_t}^{\top}\right)
e_j^{(t)}.
\label{eq:one_step_volume}
\end{equation}
Consequently, \ours{} maximizes the next residual-energy factor of the Gram determinant at every iteration and produces the same pivot sequence as modified Gram--Schmidt QRCP applied to $\mathbf X^\top$, up to tie-breaking.
\end{proposition}

\begin{proof}
In exact arithmetic, the previously selected normalized residuals form an orthonormal basis for $\operatorname{span}(\mathbf X_{\mathcal S_t})$. Decompose a candidate as $\mathbf x_j=\mathbf p_j+\mathbf r_j^{(t)}$, where $\mathbf p_j$ lies in the current selected span and $\mathbf r_j^{(t)}$ is orthogonal to it. Applying the Schur-complement identity to the augmented Gram matrix gives Eq.~\eqref{eq:one_step_volume}. Recursively applying this identity from the empty set, whose determinant is defined as one, yields Eq.~\eqref{eq:volume_product}. Maximizing the residual energy $e_j^{(t)}$ therefore maximizes the one-step augmented Gram determinant, exactly matching the pivot rule of QRCP on $\mathbf X^\top$.
\end{proof}

Proposition~\ref{prop:qrcp} provides an exact one-step characterization: each iteration maximizes the augmented Gram determinant over all candidate additions. Equivalently, GOTS selects the largest multiplicative volume factor, expressed by candidate residual energy. This greedy volume-expansion principle supplies a geometric inductive bias, while experiments evaluate how effectively the resulting subsets preserve information useful to efficient downstream VLM inference.

\section{Experiments}

\begin{table*}[t!]
\centering
{\small
\setlength{\tabcolsep}{1mm}
\begin{tabular}{lcccccccc}
\toprule
\multirow{2}{*}{\textbf{Method}}
& \textbf{GQA}
& \textbf{MMB}
& \textbf{MME}
& \textbf{POPE}
& \makecell[c]{\textbf{TextVQA}}
& \textbf{SQA}
& \textbf{OCRBench}
& \makecell[c]{\textbf{Avg.}$^{\dagger}$} \\
& \makecell[c]{Acc. $\uparrow$}
& \makecell[c]{Acc. $\uparrow$}
& \makecell[c]{P+C $\uparrow$}
& \makecell[c]{F1 $\uparrow$}
& \makecell[c]{Acc. $\uparrow$}
& \makecell[c]{Acc. $\uparrow$}
& \makecell[c]{Acc. $\uparrow$}
& $\uparrow$ \\
\midrule

\rowcolor{mygray}
\multicolumn{9}{c}{
\textit{Dynamic Resolution}
$(\mathrm{MinPix}=256\times28\times28,\quad
\mathrm{MaxPix}=2048\times28\times28)$
} \\
Avg. full-token count $\bar{T}_{\mathrm{dyn}}$
& 358.5 & 276.9 & 867.6 & 359.6 & 976.5 & 323.0 & 652.8 & -- \\
Full tokens
& 60.48 & 83.25 & 2327 & 86.16 & 82.95 & 87.46 & 83.80 & 100\% \\

\midrule
\rowcolor{mygray}
\multicolumn{9}{c}{
\textit{Token-Retention Ratio: 20\%}
\quad $\downarrow 80\%$
} \\
Avg. token budget $\bar{K}_{\mathrm{dyn}}$
& 71.7 & 55.4 & 173.5 & 71.9 & 195.3 & 64.6 & 130.6 & -- \\
VisionZip (CVPR 2025)
& 56.80 & 80.33 & 2174 & 83.38 & 65.16 & 84.23 & 59.50 & 89.5\% \\
DivPrune (CVPR 2025)
& 56.70 & 76.98 & 2163 & 80.59 & 64.75 & 80.91 & 48.10 & 85.8\% \\
CDPruner (NeurIPS 2025)
& 55.23 & 78.44 & 2136 & 78.85 & 72.72 & 83.64 & 57.90 & 88.7\% \\
MMTok (ICLR 2026)
& 58.09 & 79.30 & 2217 & 82.38 & 71.93 & 81.61 & 59.60 & 90.5\% \\
\OursRowColor \ours{}
& \textbf{58.65}
& \textbf{80.93}
& \textbf{2299}
& \textbf{85.46}
& \textbf{79.93}
& \textbf{84.73}
& \textbf{74.60}
& \textbf{96.3\%} \\

\midrule
\rowcolor{mygray}
\multicolumn{9}{c}{
\textit{Token-Retention Ratio: 10\%}
\quad $\downarrow 90\%$
} \\
Avg. token budget $\bar{K}_{\mathrm{dyn}}$
& 35.9 & 27.7 & 86.8 & 36.0 & 97.7 & 32.3 & 65.3 & -- \\
VisionZip (CVPR 2025)
& 52.47 & 75.60 & 2003 & 78.90 & 46.17 & 82.30 & 36.90 & 78.4\% \\
DivPrune (CVPR 2025)
& 53.43 & 72.85 & 1957 & 74.99 & 53.77 & 79.57 & 37.30 & 78.2\% \\
CDPruner (NeurIPS 2025)
& 51.76 & 73.88 & 1946 & 73.59 & 63.47 & \textbf{82.85} & 42.50 & 80.8\% \\
MMTok (ICLR 2026)
& 55.09 & 74.74 & 2051 & 78.75 & 59.33 & 80.47 & 43.60 & 82.3\% \\
\OursRowColor \ours{}
& \textbf{55.92}
& \textbf{77.23}
& \textbf{2180}
& \textbf{83.50}
& \textbf{75.01}
& 82.75
& \textbf{63.30}
& \textbf{90.9\%} \\

\midrule
\rowcolor{mygray}
\multicolumn{9}{c}{
\textit{Token-Retention Ratio: 5\%}
\quad $\downarrow 95\%$
} \\
Avg. token budget $\bar{K}_{\mathrm{dyn}}$
& 17.9 & 13.8 & 43.4 & 18.0 & 48.8 & 16.2 & 32.6 & -- \\
VisionZip (CVPR 2025)
& 46.28 & 67.53 & 1677 & 66.38 & 35.24 & 79.57 & 19.70 & 66.2\% \\
DivPrune (CVPR 2025)
& 49.01 & 65.89 & 1739 & 68.45 & 39.03 & 77.05 & 24.90 & 68.5\% \\
CDPruner (NeurIPS 2025)
& 47.33 & 68.13 & 1739 & 64.75 & 50.62 & 80.17 & 28.10 & 70.9\% \\
MMTok (ICLR 2026)
& 50.66 & 65.89 & 1796 & 71.35 & 44.15 & 77.19 & 30.70 & 71.6\% \\
\OursRowColor \ours{}
& \textbf{51.75}
& \textbf{72.59}
& \textbf{2009}
& \textbf{76.61}
& \textbf{67.47}
& \textbf{80.42}
& \textbf{48.10}
& \textbf{82.7\%} \\
\bottomrule
\end{tabular}
}
\caption{\textbf{Comparison on Qwen2.5-VL-7B under native dynamic resolution.}
All methods operate on identical candidate sequences with the same
per-image output budgets. Avg.$^{\dagger}$ is the mean benchmark-wise
performance retention relative to the dynamic-resolution full-token result. }
\label{tab:qwen25vl_7b}
\end{table*}

\subsection{Experimental Setup}

\paragraph{Benchmarks.}
We evaluate general visual understanding on seven primary benchmarks:
GQA~\citep{hudson2019gqa},
MMBench (MMB)~\citep{liu2024mmbench},
MME~\citep{fu2023mme},
POPE~\citep{li2023evaluating},
ScienceQA-IMG (SQA)~\citep{lu2022learn},
TextVQA (validation split)~\citep{singh2019towards},
and OCRBench~\citep{liu2024ocrbench}.
These benchmarks cover visual question answering, perception, scientific
reasoning, scene-text understanding, and optical character recognition.
We additionally evaluate chart, document, spatial, and mathematical reasoning
on
ChartQA~\citep{masry-etal-2022-chartqa},
DocVQA~\citep{Mathew_2021_WACV},
RealWorldQA~\citep{XAI_Grok1_5V_2024},
and MathVista~\citep{Lu2023MathVistaEM}.

\paragraph{Models, Baselines, and Input Settings.}
Our primary comparison uses Qwen2.5-VL-7B-Instruct under native dynamic
resolution, with $\mathrm{MinPix}=256\times28\times28$ and
$\mathrm{MaxPix}=2048\times28\times28$. Input resolution and candidate
sequence length therefore vary across images, with the token budget computed
from each image's candidate sequence. We compare \ours{} with four
representative post-encoder reduction methods:
VisionZip~\citep{yang2025visionzip},
DivPrune~\citep{alvar2025divprune},
CDPruner~\citep{zhang2025cdpruner},
and MMTok~\citep{dong2025mmtok}.
MMTok achieves the highest average retention among these baselines
(Table~\ref{tab:qwen25vl_7b}) and serves as the baseline in the
cross-backbone and additional-task studies.

We further evaluate Qwen2-VL-7B, Qwen3-VL-8B, Qwen3-VL-32B, and
InternVL3.5-8B for cross-backbone generalization. InternVL3.5-8B uses its
native dynamic-tiling configuration; the three additional Qwen models use
fixed high-resolution inputs with
$\mathrm{MinPix}=\mathrm{MaxPix}=2048\times28\times28$, which we also adopt for
the additional-task, efficiency, and ablation experiments on
Qwen2.5-VL-7B-Instruct. This controlled setting retains high-resolution inputs while reducing candidate-sequence variation, while aspect-ratio preservation and vision-grid rounding maintain model-specific visual grids across images under the specified protocol. Within every comparison, all methods receive
identical candidate sequences and output budgets, and each reduced-token result
is compared against the full-token result of the same backbone and input
setting. The cross-backbone study therefore directly evaluates whether the
advantage of \ours{} transfers across models under their specified
high-resolution input evaluation protocols.

\paragraph{Evaluation Protocol.}
We report full-token and reduced-token performance at retention ratios of
20\%, 10\%, and 5\%. For method $m$ on benchmark $b$, performance retention is
$\mathrm{Ret}_{m,b}=s_{m,b}/s_{\mathrm{full},b}\times100\%$, where $s_{m,b}$
and $s_{\mathrm{full},b}$ use the same model, input setting, and evaluation
configuration. Avg.$^{\dagger}$ is the mean of $\mathrm{Ret}_{m,b}$
over reported benchmarks, giving different metric scales equal weight.

\paragraph{Implementation Details.}
Selection indices are computed directly from vision-encoder features and used to retain the corresponding visual tokens for subsequent multimodal sequence construction and inference. We preserve model-specific boundary tokens, the original
sequence order, and the positional coordinates of the selected visual tokens.
Within each comparison, all methods share checkpoints, prompts, decoding
settings, data splits, candidate sequences, and output budgets. All
experiments use single-sample inference (batch size one) and are conducted
with \texttt{lmms-eval}~\citep{lmms_eval2024} on NVIDIA A800 GPUs
with 80\,GB memory.

\paragraph{Efficiency Measurement.}
We measure latency on 500 OCRBench samples using Qwen2.5-VL-7B-Instruct on a single A800-80\,GB GPU, with fixed-resolution inputs averaging 2,093.8 visual tokens per image. TTFT spans from vision encoding to the first generated token, covering vision encoding, projection, token selection, and LLM prefill; selector and prefill latency are measured separately to identify the source of latency changes. We exclude image preprocessing, model loading, data I/O, logging, metric computation, and post-first-token generation to isolate model-side inference costs. This yields a controlled TTFT comparison across the measured first-token inference path.

\begin{figure*}[t!]
\centering
\includegraphics[width=\textwidth]{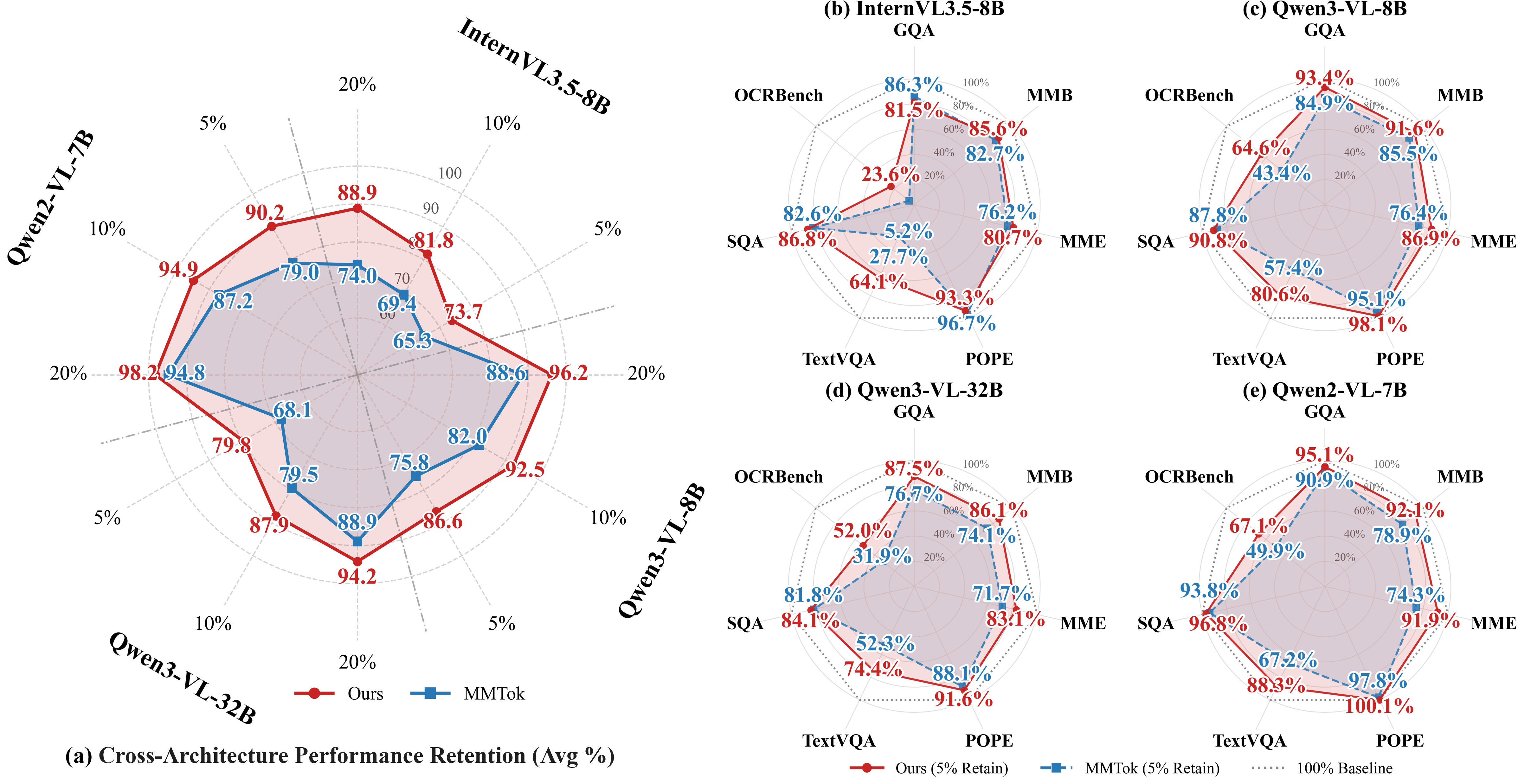}
\caption{\textbf{Cross-backbone comparison between GOTS and MMTok.}
(a) Average performance retention across token-retention ratios.
(b)--(e) Benchmark-wise retention at 5\%.
InternVL3.5-8B uses its native dynamic-tiling protocol (Dyn.); the Qwen
backbones use controlled fixed-resolution inputs (Fix.). For each backbone,
both methods share identical candidate sequences and output budgets,
normalized by the protocol-specific full-token reference.}
\label{fig:model_family}
\end{figure*}

\subsection{Main Comparison on Qwen2.5-VL-7B}

Table~\ref{tab:qwen25vl_7b} presents the primary comparison on
Qwen2.5-VL-7B-Instruct. \ours{} achieves the highest average retention
at every evaluated ratio, retaining 96.3\%, 90.9\%, and 82.7\% of full-token
performance at 20\%, 10\%, and 5\%, respectively. The corresponding results for MMTok, the strongest baseline, are 90.5\%,
82.3\%, and 71.6\%, so the margin grows from 5.8 percentage points at 20\%
retention to 8.6 at 10\% and 11.1 at 5\%. This widening margin is consistent with the selected-span motivation: smaller budgets may make avoiding redundant directions increasingly important.

At 5\% retention, \ours{} gains most over MMTok on OCRBench
(48.10 versus 30.70) and TextVQA (67.47 versus 44.15), while also improving
GQA and POPE, indicating gains across text-intensive and general visual
understanding benchmarks under aggressive compression. These results further suggest that \ours{} is particularly effective on fine-grained visual tasks under the most aggressive token budget.

\subsection{Cross-Backbone and Additional-Task Evaluation}

Figure~\ref{fig:model_family} evaluates whether the advantage of \ours{}
transfers to four backbones. Under the settings
above, \ours{} outperforms MMTok in average retention on all four at every
ratio. At 5\% retention, the margins range from 8.4 to
11.7 percentage points: 8.4 on InternVL3.5-8B, 10.8 on Qwen3-VL-8B, 11.2 on Qwen2-VL-7B, and 11.7 on Qwen3-VL-32B.
Together with Table~\ref{tab:qwen25vl_7b}, these results demonstrate that the average-retention advantage of \ours{} over MMTok transfers across multiple backbones, parameter scales, and high-resolution input protocols in our evaluation.

The benchmark-wise results at 5\% show a consistent pattern across
architectures, with especially strong gains on OCRBench and TextVQA alongside
improvements on GQA. This recurring pattern is consistent with the selected-span motivation and shows that \ours{} maintains average-retention gains across the evaluated model families, parameter scales, and resolution protocols under aggressive visual-token compression.

We further evaluate Qwen2.5-VL-7B on four additional benchmarks under the
controlled fixed-resolution protocol
(Table~\ref{tab:qwen25vl_7b_newtasks}).
\ours{} raises average retention over MMTok from 87.5\% to 90.6\% at 20\%,
from 71.5\% to 80.5\% at 10\%, and from 55.1\% to 69.2\% at 5\%.
At 5\%, the largest improvements occur on DocVQA
(50.75 versus 28.16) and ChartQA (46.08 versus 33.28), followed by
RealWorldQA (60.78 versus 53.73). Across the four tasks, the average-retention advantage becomes
increasingly pronounced as the budget decreases, showing that the advantage of \ours{} over the strongest baseline grows under aggressive compression and extends beyond the primary benchmarks to diverse chart,
document, spatial, and mathematical reasoning settings.

\begin{table}[ht]
\centering
\small
\setlength{\tabcolsep}{0.8mm}
\begin{tabular}{lccccc}
\toprule
\multirow{2}{*}{\textbf{Method}}
& \makecell[c]{\textbf{Chart}\\\textbf{QA}}
& \makecell[c]{\textbf{Doc}\\\textbf{VQA}}
& \makecell[c]{\textbf{RealWorld}\\\textbf{QA}}
& \makecell[c]{\textbf{Math}\\\textbf{Vista}}
& \makecell[c]{\textbf{Avg.}$^{\dagger}$} \\
& \makecell[c]{Acc. $\uparrow$}
& \makecell[c]{ANLS $\uparrow$}
& \makecell[c]{Acc. $\uparrow$}
& \makecell[c]{Acc. $\uparrow$}
& $\uparrow$ \\
\midrule

\rowcolor{mygray}
\multicolumn{6}{c}{
\textit{Fixed Resolution, Full-Token Reference}
} \\
Full tokens
& 76.64 & 94.88 & 68.76 & 39.80 & 100\% \\

\midrule
\rowcolor{mygray}
\multicolumn{6}{c}{
\textit{Token-Retention Ratio: 20\%}
} \\
MMTok
& 65.60
& 74.27
& 63.27
& \textbf{37.50}
& 87.5\% \\
\OursRowColor \ours{}
& \textbf{68.24}
& \textbf{80.96}
& \textbf{65.49}
& 37.00
& \textbf{90.6\%} \\

\midrule
\rowcolor{mygray}
\multicolumn{6}{c}{
\textit{Token-Retention Ratio: 10\%}
} \\
MMTok
& 50.32
& 51.82
& 58.56
& 32.10
& 71.5\% \\
\OursRowColor \ours{}
& \textbf{57.96}
& \textbf{66.94}
& \textbf{63.40}
& \textbf{33.30}
& \textbf{80.5\%} \\

\midrule
\rowcolor{mygray}
\multicolumn{6}{c}{
\textit{Token-Retention Ratio: 5\%}
} \\
MMTok
& 33.28
& 28.16
& 53.73
& 27.50
& 55.1\% \\
\OursRowColor \ours{}
& \textbf{46.08}
& \textbf{50.75}
& \textbf{60.78}
& \textbf{29.80}
& \textbf{69.2\%} \\
\bottomrule
\end{tabular}
\caption{\textbf{Controlled fixed-resolution comparison on four
additional Qwen2.5-VL-7B benchmarks.}
MMTok and \ours{} share identical candidate sequences and budgets
at each retention ratio. Avg.$^{\dagger}$ is mean benchmark-wise
retention relative to the fixed-resolution full-token result.}
\label{tab:qwen25vl_7b_newtasks}
\end{table}

\begin{table}[ht]
\centering
\small
\setlength{\tabcolsep}{1.4mm}
\begin{tabular}{lrrrr}
\toprule
\textbf{Setting}
& $\bar{K}$
& \textbf{Sel.}
& \textbf{Prefill}
& \textbf{TTFT} \\
& & \multicolumn{3}{c}{\textit{Latency (ms)}} \\
\midrule
Full tokens
& 2093.8 & -- & 170.95 & 433.16 \\
\midrule
MMTok 20\%
& 418.4 & 36.35 & 41.35 & 338.90 \\
MMTok 10\%
& 208.9 & 19.07 & 27.56 & 307.95 \\
MMTok 5\%
& 104.2 & 10.42 & 24.98 & 296.76 \\
\midrule
\ours{} 20\%
& 418.4 & 15.05 & 40.81 & 321.57 \\
\ours{} 10\%
& 208.9 & 7.57 & 27.32 & 300.96 \\
\ours{} 5\%
& 104.2 & 4.06 & 24.87 & 295.10 \\
\bottomrule
\end{tabular}
\caption{\textbf{Controlled fixed-resolution inference efficiency on
OCRBench with Qwen2.5-VL-7B.}
$\bar{K}$ is the average number of visual tokens passed to the LLM; Sel.\ is selector
overhead. TTFT spans vision encoding to the first generated token, comprising vision encoding, projection, selector overhead, and LLM prefill.}
\label{tab:inference_efficiency}
\end{table}

\subsection{Inference Efficiency}

Table~\ref{tab:inference_efficiency} reports the controlled fixed-resolution
latency comparison on OCRBench. At 5\% retention, \ours{} reduces the average visual sequence length from 2,093.8 to 104.2 tokens and lowers LLM-prefill latency
from 170.95 to 24.87\,ms, a $6.87\times$ prefill speedup.
Including its 4.06-ms selection overhead, model-side TTFT decreases
from 433.16 to 295.10\,ms, a 31.9\% reduction versus full-token
inference.

At the same budget, \ours{} and MMTok have nearly identical prefill latency
since they pass equal numbers of visual tokens to the LLM; their TTFT
difference primarily reflects selector overhead, which is lower for \ours{} at
all three budgets. The results show \ours{} converts visual-token
reduction into first-token savings across all budgets while preserving
the expected prefill benefit of shorter sequences. Its low selector overhead
enables the TTFT improvement to remain a net model-side gain, including
selection computation.

\subsection{Ablation Analysis}

Table~\ref{tab:ablation_algorithmic} examines four design choices in
\ours{}: span-aware redundancy removal, maximum-residual-energy pivoting,
feature-magnitude preservation, and empty-span initialization.
The first two groups probe the central selected-span insight; the last two examine the roles of feature magnitude and global initialization in the selection procedure. All variants use the same Qwen2.5-VL-7B backbone, fixed-resolution
candidate sequences, and a 5\% token budget, each changing one major component
of the default procedure.

\paragraph{Span-Aware Redundancy Removal.}
Feature-space non-maximum suppression (NMS) replaces span-based
orthogonalization with pairwise suppression, repeatedly selecting the
highest-norm token and removing candidates whose cosine similarity
to it exceeds $\tau=0.9$. The default \ours{} achieves higher
performance on all three benchmarks, providing evidence that
selected-span orthogonalization is preferable to this pairwise design in the
evaluated setting. The gap is consistent with measuring redundancy against the collective retained span rather than individual pivots under this fixed-threshold pairwise design.

\paragraph{Residual-Energy Pivoting.}
Keeping the orthogonal-projection update, we vary only the pivot rule.
Random Pivot samples uniformly from the candidates; Minimum Residual
Energy selects the candidate closest to the retained span. The default maximum-residual-energy rule achieves the strongest performance across all three benchmarks, supporting the importance of the pivot criterion in addition to the projection update. The degradation under minimum-residual-energy pivoting further supports the role of span expansion, rather than orthogonalization alone, among the evaluated pivot strategies in this setting.

\paragraph{Feature Magnitude and Initialization.}
Pre-$\ell_2$ Normalization scales every candidate to unit norm before
selection. Since residual energy reflects both feature magnitude and
unrepresented direction, this preprocessing removes magnitude variation and
retains only directional novelty. Using unnormalized features yields stronger
results on all three benchmarks, indicating that token norms provide
selection information in the evaluated feature space.

The Global-Average Anchor initializes the span with the mean visual feature.
The empty-span and global-average initializations achieve comparable performance across OCRBench, DocVQA, and ChartQA, indicating that \ours{} is relatively insensitive to this initialization choice within the controlled ablation setting reported here.

\begin{table}[ht]
\centering
\setlength{\tabcolsep}{0.5mm}
\small
\begin{tabular}{lccc}
\toprule
\textbf{Algorithmic Design Variant}
& \textbf{OCRBench}
& \textbf{DocVQA}
& \textbf{ChartQA} \\
\midrule

\rowcolor{gray!10}
    \multicolumn{4}{l}{
    \textit{Redundancy Removal Mechanism}
    } \\
    Feature-space NMS ($\tau=0.9$)
    & 15.90 & 13.92 & 21.08 \\

    \midrule
    \rowcolor{gray!10}
    \multicolumn{4}{l}{
    \textit{Pivot Selection Criterion}
    } \\
    Random Pivot
    & 35.90 & 27.10 & 30.24 \\
    Minimum Residual Energy
    & 4.80 & 10.21 & 12.60 \\

    \midrule
    \rowcolor{gray!10}
    \multicolumn{4}{l}{
    \textit{Feature Preconditioning}
    } \\
    Pre-$\ell_2$ Normalization
    & 29.90 & 26.05 & 29.00 \\

    \midrule
    \rowcolor{gray!10}
    \multicolumn{4}{l}{
    \textit{Global-Component Initialization}
    } \\
    Global-Average Anchor
    & 53.90 & \textbf{50.95} & \textbf{46.56} \\

    \midrule
    \ours{} (Default)
    & \textbf{54.00} & 50.75 & 46.08 \\
    \bottomrule
\end{tabular}
\caption{\textbf{Controlled ablation of GOTS design choices.}
All variants use Qwen2.5-VL-7B-Instruct, identical fixed-resolution
candidate sequences, and a 5\% token budget. Each row replaces one default
component: redundancy removal, pivot rule, feature preprocessing, or
initialization.}
\label{tab:ablation_algorithmic}
\end{table}

\section{Conclusion}

We introduced \ours{}, a training-free and query-agnostic method for reducing visual-token sequences in high-resolution VLMs. Its central perspective is selected-span complementarity: each candidate is assessed by its contribution beyond the retained span. GOTS implements this perspective by greedily selecting the token with the largest residual energy. Under nonzero pivot residuals, this rule exactly maximizes the Gram determinant among one-token augmentations of the current subset, providing a local guarantee for greedy volume expansion. Across eleven benchmarks and five high-resolution VLM backbones, GOTS improves average performance retention under aggressive compression, and controlled OCRBench measurements show a net model-side TTFT reduction after accounting for selection overhead. These findings suggest that selected-span linear complementarity is a simple and effective geometric principle for efficient visual token selection.

\bibliography{aaai2027}
\appendix

\clearpage
This appendix is organized as follows.
Section~\ref{app:cross_backbone} reports the complete benchmark-wise
results behind the cross-backbone summary in
Figure~\ref{fig:model_family}, including per-benchmark scores for
InternVL3.5-8B, Qwen3-VL-8B, Qwen3-VL-32B, and Qwen2-VL-7B at all three
retention ratios.
Section~\ref{app:algorithm} provides the full GOTS procedure in
Algorithm~\ref{alg:greedy_selection} together with implementation and
numerical-stability details omitted from the main text.

\section{Full Cross-Backbone Results}
\label{app:cross_backbone}

Figure~\ref{fig:model_family} in the main text summarizes average
retention across backbones; this section reports the underlying
per-benchmark scores. Within each table, both methods receive identical
candidate sequences and output budgets, and retention is normalized by
the full-token result of the same backbone under the same resolution
protocol, matching the evaluation protocol of the main text.

\paragraph{InternVL3.5-8B.}
Table~\ref{tab:internvl35_8b} reports results under InternVL3.5's native
dynamic-tiling protocol. GOTS improves average retention over MMTok from
74.0\% to 88.9\% at 20\%, from 69.4\% to 81.8\% at 10\%, and from 65.3\%
to 73.7\% at 5\%. The margin is smallest on this backbone, consistent
with the 8.4-point gap reported in the main text; we note that
dynamic tiling produces shorter candidate sequences than the
fixed-resolution Qwen protocol, leaving less redundancy for any
span-based criterion to exploit.

\paragraph{Qwen backbones under fixed resolution.}
Tables~\ref{tab:qwen3vl_8b}, \ref{tab:qwen3vl_32b}, and
\ref{tab:qwen2_vl_7b} report results on Qwen3-VL-8B, Qwen3-VL-32B, and
Qwen2-VL-7B under the controlled fixed-resolution protocol. At 5\%
retention, GOTS exceeds MMTok by 10.8, 11.7, and 11.2 points in average
retention, respectively, with the largest per-benchmark margins again on
OCRBench and TextVQA. The dynamic-resolution full-token rows are
reported only as auxiliary references and are not used in the retention
computation.

\section{Algorithm Details}
\label{app:algorithm}

Algorithm~\ref{alg:greedy_selection} summarizes the complete GOTS
procedure. Three details merit emphasis. First, the early-stop check
(line 7) uses tolerance $\epsilon$ on the squared residual energy;
because features are compared in squared norm, we compare
$e_{i_t}^{(t)}$ against $\epsilon^2$, so $\epsilon=10^{-6}$ corresponds
to the $10^{-12}$ threshold stated in the main text. Second, ties in the
argmax (line 6) are broken by selecting the smallest index, making the
procedure deterministic; Proposition~\ref{prop:qrcp} holds under any
tie-breaking rule. Third, selected indices are sorted back into their
original sequence order before being passed to the language model (line
13), so positional metadata remains consistent with full-token
inference; the selection order itself is only internal to the algorithm.

\begin{algorithm}[ht]
\caption{Greedy Orthogonal Token Selection (\ours{})}
\label{alg:greedy_selection}

\begin{algorithmic}[1]
\REQUIRE Visual feature matrix $\mathbf{X} \in \mathbb{R}^{N \times d}$, output token budget $K$, tolerance $\epsilon = 10^{-6}$
\ENSURE Selected and sorted token indices $\mathcal{S}_{\mathrm{out}}$
\IF{$K \leq 0$ \OR $N = 0$}
    \RETURN $\emptyset$
\ENDIF
\IF{$K \geq N$}
    \RETURN $\{1, 2, \ldots, N\}$
\ENDIF
\STATE Initialize residuals $\mathbf{r}_i^{(0)} = \mathbf{x}_i$ for all $i \in \{1, \ldots, N\}$.
\STATE Initialize available set $\mathcal{U}^{(0)} = \{1, 2, \ldots, N\}$ and selected set $\mathcal{S} = \emptyset$.
\FOR{$t = 0$ \TO $K-1$}
    \STATE $i_t \leftarrow \operatorname*{argmax}_{i \in \mathcal{U}^{(t)}} e_i^{(t)}$, where $e_i^{(t)}=\|\mathbf r_i^{(t)}\|_2^2$
    \IF{$e_{i_t}^{(t)} < \epsilon^2$}
        \STATE \textbf{break} \hfill \COMMENT{No numerically independent direction; avoid division by zero}
    \ENDIF
    \STATE $\mathcal{S} \leftarrow \mathcal{S} \cup \{i_t\}$
    \STATE $\mathcal{U}^{(t+1)} \leftarrow \mathcal{U}^{(t)} \setminus \{i_t\}$
    \STATE $\mathbf{q}_t \leftarrow \mathbf{r}_{i_t}^{(t)} / \|\mathbf{r}_{i_t}^{(t)}\|_2$
    \STATE Orthogonalize remaining residuals against $\mathbf{q}_t$:
    \STATE $\mathbf{r}_j^{(t+1)} \leftarrow \mathbf{r}_j^{(t)} - \left( \mathbf{r}_j^{(t)} \mathbf{q}_t^{\top} \right) \mathbf{q}_t, \quad \forall j \in \mathcal{U}^{(t+1)}$
\ENDFOR
\STATE $\mathcal{S}_{\mathrm{out}} \leftarrow \operatorname{sort}(\mathcal{S})$
\RETURN $\mathcal{S}_{\mathrm{out}}$
\end{algorithmic}
\end{algorithm}
\clearpage

\begin{table*}[t!]
\centering
\renewcommand{\arraystretch}{1.0} 
\setlength{\tabcolsep}{3pt}
\resizebox{\textwidth}{!}{%
\begin{tabular}{l | c c c c c c c c}
\toprule
\multirow{2}{*}{\makecell[c]{\textbf{Method}}}
& \textbf{GQA} & \textbf{MMB} & \textbf{MME} & \textbf{POPE}
& \textbf{TextVQA} & \textbf{SQA} & \textbf{OCRBench}
& \makecell[c]{\textbf{Avg.}$^{\dagger}$} \\
~
& \makecell[c]{Acc. $\uparrow$}
& \makecell[c]{Acc. $\uparrow$}
& \makecell[c]{P+C $\uparrow$}
& \makecell[c]{F1 $\uparrow$}
& \makecell[c]{Acc. $\uparrow$}
& \makecell[c]{Acc. $\uparrow$}
& \makecell[c]{Acc. $\uparrow$}
& \makecell[c]{$\uparrow$} \\
\midrule

    \rowcolor{mygray}
    \multicolumn{9}{c}{\textit{Dynamic Resolution, Full-token Reference} $\mathbf{(100\%)}$} \\
    InternVL3.5-8B
    & 61.53 & 82.98 & $1707.41{+}652.85$ & 88.61 & 76.75 & 97.17 & 82.70 & 100\% \\
    \midrule

    \rowcolor{mygray}
    \multicolumn{9}{c}{\textit{Retain 20\% $\bar{T}_{\mathrm{dyn}}$}} \\
    MMTok
    & 58.48 & 75.77 & $1568.16{+}465.71$ & 87.08 & 33.47 & 85.47 & 12.80 & 74.0\% \\

    \OursRowColor \ours{} (Ours)
    & 57.23 & 79.12 & $1614.78{+}606.07$ & 86.70 & 67.91 & 91.42 & 49.30 & \textbf{88.9\%} \\
    \midrule

    \rowcolor{mygray}
    \multicolumn{9}{c}{\textit{Retain 10\% $\bar{T}_{\mathrm{dyn}}$}} \\
    MMTok
    & 55.86 & 70.96 & $1505.26{+}459.28$ & 86.35 & 26.13 & 82.70 & 8.20 & 69.4\% \\

    \OursRowColor \ours{} (Ours)
    & 54.23 & 75.42 & $1537.62{+}505.00$ & 85.22 & 59.87 & 88.55 & 34.70 & \textbf{81.8\%} \\
    \midrule

    \rowcolor{mygray}
    \multicolumn{9}{c}{\textit{Retain 5\% $\bar{T}_{\mathrm{dyn}}$}} \\
    MMTok
    & 53.10 & 68.64 & $1414.46{+}383.21$ & 85.68 & 21.28 & 80.22 & 4.30 & 65.3\% \\

    \OursRowColor \ours{} (Ours)
    & 50.14 & 71.04 & $1427.60{+}477.14$ & 82.69 & 49.23 & 84.38 & 19.50 & \textbf{73.7\%} \\

    \bottomrule
\end{tabular}%
}
\caption{\textbf{Results on InternVL3.5-8B.}
Token reduction is performed under the dynamic-resolution setting. Therefore, the dynamic-resolution full-token result is used as the reference for computing Avg.$^{\dagger}$. Avg.$^{\dagger}$ is computed over all seven benchmarks, with MME measured by the sum of perception and cognition scores.}
\label{tab:internvl35_8b}
\end{table*}

\begin{table*}[t!]
\centering
\renewcommand{\arraystretch}{1.0} 
\setlength{\tabcolsep}{3pt}
\resizebox{\textwidth}{!}{%
\begin{tabular}{l | c c c c c c c c}
\toprule
\multirow{2}{*}{\makecell[c]{\textbf{Method}}}
& \textbf{GQA} & \textbf{MMB} & \textbf{MME} & \textbf{POPE}
& \textbf{TextVQA} & \textbf{SQA} & \textbf{OCRBench}
& \makecell[c]{\textbf{Avg.}$^{\dagger}$} \\
~
& \makecell[c]{Acc. $\uparrow$}
& \makecell[c]{Acc. $\uparrow$}
& \makecell[c]{P+C $\uparrow$}
& \makecell[c]{F1 $\uparrow$}
& \makecell[c]{Acc. $\uparrow$}
& \makecell[c]{Acc. $\uparrow$}
& \makecell[c]{Acc. $\uparrow$}
& \makecell[c]{$\uparrow$} \\
\midrule

\rowcolor{mygray}
\multicolumn{9}{c}{\textit{Dynamic Resolution} $(\mathrm{MinPix}=256\times28\times28,\ \mathrm{MaxPix}=2048\times28\times28)$, \textit{Full-token Comparison}} \\
Qwen3-VL-8B
& 61.59 & 84.54 & $1730.11{+}653.93$ & 88.42 & 81.62 & 94.60 & 88.10 & 99.8\% \\
\midrule

\rowcolor{mygray}
\multicolumn{9}{c}{\textit{Fixed Resolution} $(\mathrm{MinPix}=\mathrm{MaxPix}=2048\times28\times28)$, \textit{Full-token Reference} $\mathbf{(100\%)}$} \\
Qwen3-VL-8B
& 62.08 & 85.65 & $1751.63{+}659.64$ & 89.84 & 82.58 & 93.65 & 85.20 & 100\% \\
\midrule

\rowcolor{mygray}
\multicolumn{9}{c}{\textit{Retain 20\% $\bar{T}_{\mathrm{fix}}$}} \\
MMTok
& 57.85 & 81.95 & $1671.33{+}518.92$ & 87.87 & 67.70 & 87.46 & 57.60 & 88.6\% \\
\OursRowColor \ours{} (Ours)
& 61.27 & 83.33 & $1685.35{+}629.64$ & 89.99 & 78.45 & 90.63 & 76.40 & \textbf{96.2\%} \\
\midrule

\rowcolor{mygray}
\multicolumn{9}{c}{\textit{Retain 10\% $\bar{T}_{\mathrm{fix}}$}} \\
MMTok
& 56.50 & 78.95 & $1552.43{+}433.92$ & 86.54 & 55.97 & 86.56 & 44.40 & 82.0\% \\
\OursRowColor \ours{} (Ours)
& 60.41 & 81.95 & $1668.39{+}569.64$ & 89.22 & 74.46 & 88.15 & 66.80 & \textbf{92.5\%} \\
\midrule

\rowcolor{mygray}
\multicolumn{9}{c}{\textit{Retain 5\% $\bar{T}_{\mathrm{fix}}$}} \\
MMTok
& 52.73 & 73.19 & $1452.21{+}391.07$ & 85.43 & 47.39 & 82.20 & 37.00 & 75.8\% \\
\OursRowColor \ours{} (Ours)
& 58.01 & 78.43 & $1598.02{+}496.78$ & 88.14 & 66.60 & 85.03 & 55.00 & \textbf{86.6\%} \\
\bottomrule
\end{tabular}%
}
\caption{\textbf{Results on Qwen3-VL-8B.}
Token reduction is performed under the fixed-resolution setting. Therefore, the fixed-resolution full-token result is used as the reference for computing Avg.$^{\dagger}$. The dynamic-resolution full-token result is reported only as an auxiliary comparison. Avg.$^{\dagger}$ is computed over all seven benchmarks, with MME measured by the sum of perception and cognition scores.}
\label{tab:qwen3vl_8b}
\end{table*}

\begin{table*}[t!]
\centering
\renewcommand{\arraystretch}{1.0} 
\setlength{\tabcolsep}{3pt}
\resizebox{\textwidth}{!}{%
\begin{tabular}{l | c c c c c c c c}
\toprule
\multirow{2}{*}{\makecell[c]{\textbf{Method}}}
& \textbf{GQA} & \textbf{MMB} & \textbf{MME} & \textbf{POPE}
& \textbf{TextVQA} & \textbf{SQA} & \textbf{OCRBench}
& \makecell[c]{\textbf{Avg.}$^{\dagger}$} \\
~
& \makecell[c]{Acc. $\uparrow$}
& \makecell[c]{Acc. $\uparrow$}
& \makecell[c]{P+C $\uparrow$}
& \makecell[c]{F1 $\uparrow$}
& \makecell[c]{Acc. $\uparrow$}
& \makecell[c]{Acc. $\uparrow$}
& \makecell[c]{Acc. $\uparrow$}
& \makecell[c]{$\uparrow$} \\
\midrule

\rowcolor{mygray}
\multicolumn{9}{c}{\textit{Dynamic Resolution} $(\mathrm{MinPix}=256\times28\times28,\ \mathrm{MaxPix}=2048\times28\times28)$, \textit{Full-token Comparison}} \\
Qwen3-VL-32B
& 61.62 & 88.83 & $1790.85{+}705.71$ & 89.42 & 82.85 & 96.83 & 88.50 & 100.1\% \\
\midrule

\rowcolor{mygray}
\multicolumn{9}{c}{\textit{Fixed Resolution} $(\mathrm{MinPix}=\mathrm{MaxPix}=2048\times28\times28)$, \textit{Full-token Reference} $\mathbf{(100\%)}$} \\
Qwen3-VL-32B
& 62.00 & 89.00 & $1791.99{+}690.71$ & 90.20 & 84.23 & 97.17 & 85.20 & 100\% \\
\midrule

\rowcolor{mygray}
\multicolumn{9}{c}{\textit{Retain 20\% $\bar{T}_{\mathrm{fix}}$}} \\
MMTok
& 58.24 & 82.90 & $1730.30{+}623.93$ & 87.04 & 68.84 & 89.49 & 59.50 & 88.9\% \\
\OursRowColor \ours{} (Ours)
& 60.07 & 85.05 & $1734.37{+}643.93$ & 88.43 & 77.92 & 92.17 & 72.80 & \textbf{94.2\%} \\
\midrule

\rowcolor{mygray}
\multicolumn{9}{c}{\textit{Retain 10\% $\bar{T}_{\mathrm{fix}}$}} \\
MMTok
& 53.78 & 76.55 & $1567.39{+}430.36$ & 85.00 & 61.19 & 82.99 & 43.70 & 79.5\% \\
\OursRowColor \ours{} (Ours)
& 57.91 & 82.30 & $1662.00{+}568.21$ & 87.18 & 71.99 & 87.11 & 57.90 & \textbf{87.9\%} \\
\midrule

\rowcolor{mygray}
\multicolumn{9}{c}{\textit{Retain 5\% $\bar{T}_{\mathrm{fix}}$}} \\
MMTok
& 47.53 & 65.98 & $1391.51{+}388.57$ & 79.51 & 44.05 & 79.52 & 27.20 & 68.1\% \\
\OursRowColor \ours{} (Ours)
& 54.25 & 76.63 & $1562.83{+}501.07$ & 82.62 & 62.65 & 81.76 & 44.30 & \textbf{79.8\%} \\
\bottomrule
\end{tabular}%
}
\caption{\textbf{Results on Qwen3-VL-32B.}
Token reduction is performed under the fixed-resolution setting. Therefore, the fixed-resolution full-token result is used as the reference for computing Avg.$^{\dagger}$. The dynamic-resolution full-token result is reported only as an auxiliary comparison. Avg.$^{\dagger}$ is computed over all listed benchmarks, with MME measured by the sum of perception and cognition scores.}
\label{tab:qwen3vl_32b}
\end{table*}

\begin{table*}[t!]
\centering
\renewcommand{\arraystretch}{1.0} 
\setlength{\tabcolsep}{3pt}
\resizebox{\textwidth}{!}{%
\begin{tabular}{l | c c c c c c c c}
\toprule
\multirow{2}{*}{\makecell[c]{\textbf{Method}}}
& \textbf{GQA} & \textbf{MMB} & \textbf{MME} & \textbf{POPE}
& \textbf{TextVQA} & \textbf{SQA} & \textbf{OCRBench}
& \makecell[c]{\textbf{Avg.}$^{\dagger}$} \\
~
& \makecell[c]{Acc. $\uparrow$}
& \makecell[c]{Acc. $\uparrow$}
& \makecell[c]{P+C $\uparrow$}
& \makecell[c]{F1 $\uparrow$}
& \makecell[c]{Acc. $\uparrow$}
& \makecell[c]{Acc. $\uparrow$}
& \makecell[c]{Acc. $\uparrow$}
& \makecell[c]{$\uparrow$} \\
\midrule

\rowcolor{mygray}
\multicolumn{9}{c}{\textit{Dynamic Resolution} $(\mathrm{MinPix}=256\times28\times28,\ \mathrm{MaxPix}=2048\times28\times28)$, \textit{Full-token Comparison}} \\
Qwen2-VL-7B
& 62.38 & 79.38 & $1702.24{+}655.00$ & 87.77 & 81.95 & 85.15 & 85.30 & 100.5\% \\
\midrule

\rowcolor{mygray}
\multicolumn{9}{c}{\textit{Fixed Resolution} $(\mathrm{MinPix}=\mathrm{MaxPix}=2048\times28\times28)$, \textit{Full-token Reference} $\mathbf{(100\%)}$} \\
Qwen2-VL-7B
& 61.85 & 80.33 & $1649.47{+}656.43$ & 88.55 & 83.10 & 85.66 & 81.70 & 100\% \\
\midrule

\rowcolor{mygray}
\multicolumn{9}{c}{\textit{Retain 20\% $\bar{T}_{\mathrm{fix}}$}} \\
MMTok
& 60.03 & 76.63 & $1626.59{+}550.71$ & 88.22 & 78.42 & 83.47 & 69.60 & 94.8\% \\
\OursRowColor \ours{} (Ours)
& 61.18 & 78.61 & $1663.75{+}638.21$ & 89.20 & 81.40 & 84.46 & 76.20 & \textbf{98.2\%} \\
\midrule

\rowcolor{mygray}
\multicolumn{9}{c}{\textit{Retain 10\% $\bar{T}_{\mathrm{fix}}$}} \\
MMTok
& 58.73 & 70.02 & $1467.88{+}441.43$ & 88.21 & 70.24 & 82.22 & 53.60 & 87.2\% \\
\OursRowColor \ours{} (Ours)
& 60.48 & 77.92 & $1632.02{+}601.43$ & 88.24 & 78.55 & 83.33 & 66.20 & \textbf{94.9\%} \\
\midrule

\rowcolor{mygray}
\multicolumn{9}{c}{\textit{Retain 5\% $\bar{T}_{\mathrm{fix}}$}} \\
MMTok
& 56.23 & 63.40 & $1321.45{+}391.79$ & 86.58 & 55.84 & 80.38 & 40.80 & 79.0\% \\
\OursRowColor \ours{} (Ours)
& 58.84 & 73.97 & $1551.22{+}566.79$ & 88.61 & 73.38 & 82.95 & 54.80 & \textbf{90.2\%} \\
\bottomrule
\end{tabular}%
}
\caption{\textbf{Performance Comparison on Qwen2-VL-7B.}
Token reduction is performed under the fixed-resolution setting. Therefore, the fixed-resolution full-token result is used as the reference for computing Avg.$^{\dagger}$. The dynamic-resolution full-token result is reported only as an auxiliary comparison. Avg.$^{\dagger}$ is computed over all seven benchmarks, with MME measured by the sum of perception and cognition scores. MMTok is used as a baseline, and \ours{} denotes our method.}
\label{tab:qwen2_vl_7b}
\end{table*}


\end{document}